\documentclass[12pt]{article}
\usepackage{fullpage}

\usepackage[utf8]{inputenc} %
\usepackage[T1]{fontenc}    %
\usepackage{hyperref}       %
\usepackage{url}            %
\usepackage{booktabs}       %
\usepackage{amsfonts}       %
\usepackage{nicefrac}       %
\usepackage{microtype}      %

\usepackage{graphicx}
\usepackage{subfigure}
\usepackage{amsmath}
\usepackage{amsthm}
\usepackage{wrapfig}
\usepackage{color}
\usepackage[numbers]{natbib}

\theoremstyle{plain}
\newtheorem{prop}{\protect\propositionname}
\theoremstyle{plain}
\newtheorem{thm}{\protect\theoremname}
\theoremstyle{plain}
\newtheorem{assumption}{\protect\assumptionname}
\theoremstyle{plain}

\theoremstyle{plain}

\providecommand{\assumptionname}{Assumption}
\providecommand{\lemmaname}{Lemma}
\providecommand{\propositionname}{Proposition}
\providecommand{\corollaryname}{Corollary}
\providecommand{\theoremname}{Theorem}

\graphicspath{{figures/}}

\title{Posterior Ratio Estimation of Latent Variables}
\author{Song Liu, Yulong Zhang, Mingxuan Yi\footnote{University of Bristol, UK, emails: \{song.liu, yulong.zhang, mingxuan.yi\}@bristol.ac.uk}, Mladen Kolar\footnote{ University of Chicago, Booth School of Business, US, email: mkolar@chicagobooth.edu}}

\newcommand{\argmin}{\mathop{\rm argmin}\limits}

\newcommand{\phatp}{\hat{p}(\boldy_p)}

\newcommand{\infHessLL}{\lambda_{\min}\left[\nabla_\bolddelta^2\ell(\bolddelta)\right]}
\newcommand{\hessianLLstar}{\nabla_\bolddelta^2 \ell(\bolddelta^*)}
\newcommand{\hessianLL}{\nabla_\bolddelta^2 \ell(\bolddelta)}

\newcommand{\boldtheta}{{\boldsymbol{\theta}}}
\newcommand{\bolddelta}{{\boldsymbol{\delta}}}

\newcommand{\bolda}{{\boldsymbol{a}}}

\newcommand{\boldA}{{\boldsymbol{A}}}

\newcommand{\boldmu}{{\boldsymbol{\mu}}}

\newcommand{\boldf}{{\boldsymbol{f}}}

\newcommand{\boldone}{{\boldsymbol{1}}}

\newcommand{\boldSigma}{{\boldsymbol{\Sigma}}}

\newcommand{\boldb}{{\boldsymbol{b}}}

\newcommand{\mathbbR}{\mathbb{R}}

\newcommand{\boldx}{{\boldsymbol{x}}}

\newcommand{\boldz}{{\boldsymbol{z}}}
\newcommand{\boldg}{{\boldsymbol{g}}}

\newcommand{\boldzero}{{\boldsymbol{0}}}

\newcommand{\boldy}{{\boldsymbol{y}}}

\newcommand{\boldI}{{\boldsymbol{I}}}

\newcommand{\mathbbE}{\mathbb{E}}

\newcommand{\vertiii}[1]{{\left\vert\kern-0.25ex\left\vert\kern-0.25ex\left\vert #1 
    \right\vert\kern-0.25ex\right\vert\kern-0.25ex\right\vert}}
\begin{document}
\date{}
\maketitle

\begin{abstract}
Density Ratio Estimation has attracted attention from machine learning community due to its ability of comparing the underlying distributions of two datasets. However, in some applications, we want to compare distributions of random variables that are \emph{inferred} from observations. In this paper, we study the problem of estimating the ratio between two posterior probability density functions of a latent variable. Particularly, we assume the posterior ratio function can be well-approximated by a parametric model, which is then estimated using observed information and prior samples. We prove consistency of our estimator and the asymptotic normality of the estimated parameters as the number of prior samples tending to infinity. Finally, we validate our theories using numerical experiments and demonstrate the usefulness of the proposed method through some real-world applications.  
\end{abstract}

\section{Introduction}
Comparing the underlying distributions of two given datasets has been an important task and has a wide range of applications. 
For example, change detection algorithms \citep{Kawahara2012} compare datasets collected at different time points and report how the underlying distribution has shifted over time;
transfer learning \citep{quionero2009dataset} uses the estimated differences between two datasets to efficiently share information between different tasks. 
Generative Adversarial Net (GAN) \citep{Goodfellow2014} learns an implicit generative model whose output minimizes the differences between an artificial dataset and a real dataset. 

Various computational methods have been proposed for comparing underlying distributions given two datasets. For example, Maximum Mean Discrepancy (MMD) \citep{gretton2012kernel} computes the distance between the kernel mean embeddings of two datasets in Reproducing Kernel Hilbert Space (RKHS). Density ratio estimation (DRE) \citep{Sugiyama2012} estimates the ratio function between two probability densities. 
The estimated ratio can be used to approximate various statistical divergences \citep{Nowozin2016,kanamori2011f}.
Recently, Wasserstein distance (and Optimal Transport distance) as an alternative to statistical divergence has been explored for many learning tasks \citep{frogner2015learning,Arjovsky17WGAN} and efficient algorithms for computing such a distance from two sets of samples have been proposed \citep{genevay2016stochastic}. 

Sometimes we are not
interested in comparing distributions of observable random variables, but 
distributions of \emph{latent} random variables given observations.
For example, we want to monitor the engine vibration pattern via readings from a sensor placed on the engine. However, our readings may  be noisy as the sensor is also susceptible to vibrations of other nearby components. Thus, instead of directly monitoring the observed sensor readings, it makes sense to compare the latent vibration patterns given such readings. In the classic Bayesian setting, a latent pattern is described by a \emph{posterior probability}, which can be inferred from observations, a prior, and a likelihood. 
Thus, comparing distributions of latent variables can be framed as a comparison between
two posteriors. In this paper, we consider the problem of  estimating the \emph{ratio} between two posteriors probabilities.

One straightforward approach to obtain the posterior ratio is to represent
two posterior probabilities in the form of  likelihood functions and priors using Bayes' rule. However, in many cases, we do not have explicit expressions of priors, but only some simulated samples from the priors. 
In the above example, coming up with an explicit prior for the engine vibration is almost impossible as it is governed by a complex physics process. However, it is possible to simulate samples from the latent variable under different settings (e.g., ``normal condition'' or ``damaged engine''), either by conducting experiments or running computer simulations.

One simple solution to the ``unknown prior'' problem above is 
to estimate prior densities from prior samples and then plug the estimates into the posterior ratio formula.  However, it is  unlikely that we know how to model 
the prior distribution, while non-parametric estimators such as kernel density estimators (KDE) \citep{rosenblatt1956,parzen1962estimation} become infeasible when the dimensionality is high. Alternatively, one can estimate the ratio of priors directly using DRE. 
To obtain the posterior ratio,
we only need the \emph{ratio of priors}, thus we should not try to solve two separate and more generic estimation problems individually. However, this likelihood-agnostic approach may not be optimal as it cannot guarantee that the estimated ratio would behave well with respect to the given likelihood functions and observations which are simply ignored during the estimation.

Following this rationale, we develop a novel algorithm which estimates the ratio between two posterior probabilities directly. The algorithm is referred as Posterior Ratio Estimation (PRE) and can be regarded as an analogue of DRE for posterior probabilities. We assume two likelihood functions and two sets of prior samples are given. Our algorithm approximates the true posterior ratio by minimizing the KL divergence from one posterior to the other posterior reweighed by a ratio model. 
We prove that the estimated model parameters eventually converge to the true minimizer of the KL divergence, as the number of the prior samples increases. The estimated parameters are asymptotically normal, which is useful for statistical inference. 
We evaluate the performance of PRE in two applications: 
Latent signal detection and local linear classifier extraction.
Promising results are obtained.

\section{Posterior Ratio Estimation}
\subsection{Problem Setting}
Formally, $\mathbb{P}$ and $\mathbb{Q}$ are two joint probability distributions of two random variables $Y$ and $X$. We shorten the marginal/conditional probability of $\mathbb{P}$ and $\mathbb{Q}$ as $p(\boldx) := \mathbb{P}(X = \boldx), p(\boldx|\boldy) := \mathbb{P}(X = \boldx|Y = \boldy)$, etc.   
Suppose we have 
\begin{itemize}
    \item $\boldy_p, \boldy_q$: two observations from $\mathbb{P}(Y)$ and $\mathbb{Q}(Y)$ respectively;
    \item $p(\boldy_p|\boldx)$, $q(\boldy_q|\boldx)$: two likelihood functions;
    \item $ X_p := \{\boldx_p^{(i)}\}_{i=1}^{n_p}, X_q := \{\boldx_q^{(i)}\}_{i=1}^{n_q}$: two sets of samples from  $\mathbb{P}(X)$ and $\mathbb{Q}(X)$ respectively.
\end{itemize}
We want to estimate the posterior density ratio
 $\frac{p(\boldx|\boldy_p)}{q(\boldx|\boldy_q)}$ up to a constant that is independent of $\boldx$.

\subsection{Posterior Ratio Model}
We model the posterior ratio $\frac{p(\boldx|\boldy_p)}{q(\boldx|\boldy_q)}$ using a parametric function, 
\begin{align}
\label{eq.ratio.model}
    r(\boldx; \bolddelta) := \frac{\exp\langle \bolddelta, \boldf(\boldx)\rangle}{Z(\bolddelta)}, 
    \quad Z(\bolddelta) := \int q(\boldx|\boldy_q) \exp\langle \bolddelta, \boldf(\boldx) \rangle d\boldx 
\end{align}
and $\boldf$ is a feature function chosen beforehand. For example, the polynomial feature, $\boldf(\boldx) = [ \boldx, \boldx^2]^\top$ (where the power is applied over all elements of $\boldx$),
is shown to have good performance in our experiments. 
$Z(\bolddelta)$ ensures that the ratio model is properly normalized: for any $\bolddelta$,
$    
\int q(\boldx|\boldy_q) r(\boldx; \bolddelta) d\boldx = 1.
$
Note that 
if both $p(\boldx|\boldy_p)$ and $q(\boldx|\boldy_q)$ are from the exponential family with the sufficient statistic $\boldf$, 
then this density ratio model is well-specified.

\subsection{Posterior Ratio Estimation}
Our estimation strategy is similar to the one used in a density ratio estimator called Kullback-Leibler Importance Estimation Procedure (KLIEP) \citep{Sugiyama2008,Tsuboi2009}. 
To estimate $\bolddelta$ in the ratio model $r(\boldx; \bolddelta)$, 
we minimize the {\rm KL} divergence from $p(\boldx|\boldy_p)$ to $r(\boldx; \bolddelta)q(\boldx|\boldy_q)$: 
\begin{align*}
    \mathrm{KL}[p|r\cdot q] 
    := & \int p(\boldx|\boldy_p) \log \frac{p(\boldx|\boldy_p)}{q(\boldx|\boldy_q)r(\boldx; \bolddelta)} d\boldx \\
    = & C - \int p(\boldx|\boldy_p) \log r(\boldx; \bolddelta) d\boldx 
    =  C - \frac{1}{p(\boldy_p)}\int p(\boldy_p|\boldx)p(\boldx) \log r(\boldx; \bolddelta) d\boldx, 
\end{align*}
where $C$ and $p(\boldy_p)$ do not depend on $\bolddelta$ and the last equality is due to Bayes' rule.
Ignoring constants that are not dependent on $\bolddelta$, we obtain the following minimization problem:
\begin{align} 
 \min_{\bolddelta} -\int p(\boldy_p|\boldx) p(\boldx) \log r(\boldx;\bolddelta) d\boldx \label{eq.obj.2}. 
\end{align} 

From now, we shorten the likelihood functions $p(\boldy_p|\boldx),q(\boldy_q|\boldx)$ as $l_p(\boldx),l_q(\boldx)$.
After replacing $r(\boldx;\bolddelta)$ with its definition in \eqref{eq.ratio.model}, and applying the Bayes' rule on $q(\boldx|\boldy_q)$, 
we obtain the following optimization problem:
\begin{align}
        \min_{\bolddelta} \mathbb{E}_{p(\boldx)} \left[ l_p(\boldx) \cdot \langle -\bolddelta, \boldf(\boldx) \rangle \right] 
    + \mathbbE_{p(\boldx)} \left[l_p(\boldx) \right] \cdot \log \mathbbE_{q(\boldx)} \left[\frac{l_q(\boldx)}{q(\boldy_q)} \cdot \exp\langle \bolddelta, \boldf(\boldx) \rangle\right],
    \label{eq.obj}
\end{align}
where $\mathbb{E}_{p(\boldx)}[f(\boldx)] := \int p(\boldx) f(\boldx) d\boldx$.
\begin{prop}
\label{prop.global.mini}
    If $p(\boldy_p)>0$,
	a $\bolddelta^*$ is a global minimizer of \eqref{eq.obj.2} if and only if 
	\begin{align}
	 \label{eq.optim.obj}
    \mathbbE_{p(\boldx|\boldy_p)} \left[ \boldf(\boldx)\right] 
    = \mathbbE_{q(\boldx|\boldy_q)}\left[ r(\boldx; \bolddelta^*)\boldf(\boldx)\right]. 
	\end{align}
\end{prop}
\begin{proof}
We can verify this statement by taking the derivative of the objective function in \eqref{eq.obj} with respect to $\bolddelta$ and setting it to zero. Then applying the equality $\mathbbE_{p(\boldx)} [l_p(\boldx)] = p(\boldy_p)$ and Bayes' rule 
yields the desired result. The global minimum is guaranteed by the convexity of \eqref{eq.obj}. 
\end{proof}

Proposition \ref{prop.global.mini} tells us that an optimal parameter of our posterior ratio model should reweight the expectation of $\boldf$ with respect to $q(\boldx|\boldy)$ such that it equals the expectation of $\boldf$ with respect to $p(\boldx|\boldy)$, which is reasonable, considering $r$ is a model for $p/q$. 

Further, we can derive a tractable empirical objective function of  \eqref{eq.obj.2}:
Denote 
$
	\hat{\mathbbE}_{p_{\boldx}} \left[f(\boldx)\right] := \frac{1}{n_p} \sum_{i=1}^{n_p} f(\boldx^{(i)}_p), 
	~~~ \hat{p}(\boldy_p) := \hat{\mathbbE}_{p(\boldx)} \left[l_p(\boldx) \right].
$ After replacing the expectations, $p(\boldy)$
 and $q(\boldy)$ in \eqref{eq.obj} with sample averages, we obtain the estimator for the ratio parameter $\bolddelta$: 
\begin{align}
    \label{eq.emp.obj}
    \hat{\bolddelta} := \argmin_\bolddelta    -\hat{\mathbb{E}}_{p(\boldx)} \left[ l_p(\boldx) \cdot \langle \bolddelta, \boldf(\boldx) \rangle \right] 
    +\hat{p}(\boldy_p) \cdot \log \hat{\mathbbE}_{q(\boldx)} \left[l_q(\boldx) \cdot \exp\langle \bolddelta, \boldf(\boldx) \rangle\right] + C', 
\end{align}
where $C' = -\hat{p}(\boldy_p) \log \hat{q}(\boldy_q)$ and is independent of $\bolddelta$. 
Note that if $l_p = 1$ and $l_q = 1$, 
then $\hat{\bolddelta}$ coincides with the KLIEP estimator of $\frac{p(\boldx)}{q(\boldx)}$. 
The above optimization problem is \emph{convex} with respect to $\bolddelta$ 
and the gradient descent algorithm can be employed to find $\hat{\bolddelta}$. 

\subsection{Relationship with Density Ratio Estimation}
DRE \citep{Sugiyama2012,Nguyen2008} approximates the ratio function between two probability densities 
given two sets of samples, which are drawn from these probability distributions.
As we previously mentioned, 
one of the DRE algorithms, KLIEP shares the same KL divergence minimization criterion with PRE. However,
DRE learns a ratio of probability densities of two \emph{observed} random variables. If one wants to estimate $\frac{p(\boldx|\boldy_p)}{q(\boldx|\boldy_q)}$ using KLIEP, one must directly draw samples from $p(\boldx|\boldy_p)$ and $q(\boldx|\boldy_q)$, which themselves are difficult problems. 
In comparison, PRE avoids this problem by using Bayes' rule in the objective, and estimates the ratio of two \emph{latent} probability densities given likelihood functions and samples from their priors without needing to sample from the posterior distributions. 

However, there is a ``plugin'' method to obtain posterior ratio by using DRE rather than PRE: We can obtain  $\frac{p(\boldx)}{q(\boldx)}$ by performing the regular DRE using datasets $X_p$ and $X_q$,  then multiply it by the ratio of likelihood  $\frac{p(\boldy_p|\boldx)}{q(\boldy_q|\boldx)}$. Due to Bayes' rule, this estimator approximates the true posterior ratio up to a constant that is independent of $\boldx$. We refer this estimator as the ``plugin'' estimator since rather than estimating the posterior ratio, we ``plugin'' the estimate of prior ratio and likelihood to obtain a new estimator. Its performance is compared in Section \ref{sec.latent.detection}. 

\section{Consistency and Asymptotic Normality of $\hat{\bolddelta}$}
First  we establish the consistency of  $\hat{\bolddelta}$ as an estimator of the posterior ratio parameters and analyze its sufficient conditions. 
Under the assumption $\infty >\hat{p}(\boldy_p)>0$, let us denote the optimization problem \eqref{eq.emp.obj} as 
$
    \argmin_{\bolddelta} \ell(\bolddelta) \cdot \hat{p}(\boldy_p),
$
where $\ell(\bolddelta)$ is simply the objective in \eqref{eq.emp.obj} rescaled by $1/\hat{p}(\boldy_p)$. 
\paragraph{Notations.} $\|\bolda\|$ is the $\ell_2$ norm of a vector $\bolda$. $\|\boldA\|$ is the spectral norm of a matrix $\boldA$. $\mathrm{Ball}(\bolda, R)$ is a $\ell_2$ ball centered at $\bolda$ with radius $R$. $\nabla_\bolda f(\bolda_0)$ is the gradient of $f(\bolda)$ evaluated at $\bolda_0$, while $\nabla_\bolda^2 f(\bolda)$ is the Hessian of $f$. 
$\lambda_\mathrm{min}(\boldA)$ is the minimum eigenvalue of a matrix $\boldA$. 
$\overset{\mathbb{P}}{\to} $ and $ \rightsquigarrow$ denote convergence in probability and in distribution, respectively. 

\subsection{Consistency of $\hat{\bolddelta}$ and Its Sufficient Conditions}

We first state a main theorem specifying a list of sufficient conditions under which our estimator $\hat{\bolddelta}$ converges to $\bolddelta^*$ in probability as $n_p \wedge n_q \to \infty$. 
Then we show these conditions are satisfied with high probability when similar conditions are imposed on some population quantities. 

\begin{assumption}
\label{ass.correctmodel}
$\bolddelta^*$, which is the minimizer of \eqref{eq.obj}, is unique. 
\end{assumption}
Since \eqref{eq.obj} is convex and twice continuously differentiable with respect to $\bolddelta$, a \emph{sufficient} condition of Assumption \ref{ass.correctmodel} is that the Hessian of \eqref{eq.obj} is positive definite for all $\bolddelta$.

\begin{thm}
\label{them.2}
Suppose Assumption \ref{ass.correctmodel} holds,
\begin{align}
\label{eq.optimgrad.converge}
    \exists R_1&>0, \|\nabla_\bolddelta \ell(\bolddelta^*)\| \le \frac{R_1}{\sqrt{n_p\wedge n_q}}, 
\end{align} with probability at least $\epsilon_{R_1}$ and $\exists R_2>0, C_r>1, $
\begin{align}
\label{eq.mineig}
\inf_{\bolddelta\in \Delta_{n}}\lambda_{\min}\left[\nabla^2_\bolddelta \ell(\bolddelta)\right]  > R_2,  \Delta_{n}:= \mathrm{Ball}\left(\bolddelta^*,\frac{R_1}{R_2} \cdot \frac{C_r}{{\sqrt{n_p\wedge n_q}}}\right),
\end{align}
with probability at least $1-\epsilon_{R_2}$. 
Then $\exists N$ such that $ n_p\wedge n_q > N$,  we have $\hat{\bolddelta}    \overset{\mathbb{P}}{\to}\bolddelta^*$ with probability $1 - \epsilon_{R_1} - \epsilon_{R_2}$.
\end{thm}
The proof can be found in the supplementary materials. 
Although Theorem \ref{them.2} is established on a set of inequalities of random variables, we show these sample conditions hold with high probability when their \emph{population counterparts} hold. 

\subsection{Eigenvalue Lowerbound of $\hessianLL$}
\label{sec.eig.hessian}
Let us define a few new notations: 
\[
\boldf_p'(\boldx) := \frac{l_p(\boldx) }{p(\boldy_p)} \boldf(\boldx), \quad \boldf_q'(\boldx) := \frac{l_q(\boldx) }{q(\boldy_q)} \boldf(\boldx) \quad r_{n_q} := \frac{\exp\langle \bolddelta^\top \boldf(\boldx) \rangle}{\hat{\mathbbE}_{q(\boldx)}\left[l_q(\boldx) \cdot \exp\langle \bolddelta^\top \boldf(\boldx) \rangle\right]/q(\boldy_q)},\]
where $r_{n_q}$ can be seen as the empirical version of our posterior ratio model. 

Now we analyze sufficient conditions \eqref{eq.mineig} in Theorem \eqref{them.2}, which states that the minimum eigenvalue of $\hessianLL$ is lower-bounded within the neighbourhood of $\bolddelta^*$. On one hand, it is easy to see that if there exists a $\boldSigma \in \mathbb{R}^{d\times d}$, $\boldSigma \approx \nabla^2_\bolddelta \ell(\bolddelta), \forall \bolddelta \in \Delta_n$, then
$\inf_{\bolddelta\in \Delta_{n}}\lambda_{\min}\left[\nabla_\bolddelta^2\ell(\bolddelta)\right] \approx \lambda_{\min}\left[\boldSigma\right]$. More formally, Wely's inequality \citep{Horn1986} allows us to lower-bound $\inf_{\bolddelta\in \Delta_{n}}\infHessLL$:
\begin{align}
\label{eq.lower.hessian}
    \inf_{\bolddelta\in \Delta_{n}} \infHessLL 
\ge \lambda_{\min}\left[\boldSigma\right] -  \sup_{\bolddelta \in \Delta_{n}} \|\boldSigma - \nabla_\bolddelta^2\ell(\bolddelta)\|. 
\end{align}

On the other hand, it can be seen that \[\nabla^2_\bolddelta \ell(\bolddelta)=  \hat{\mathbbE}_{q(\boldx)}\left[ r_{n_q}(\boldx; \bolddelta) \boldf_q'(\boldx) \boldf_q'(\boldx)^\top\right] - \hat{\mathbbE}_{q(\boldx)}\left[ r_{n_q}(\boldx; \bolddelta) \boldf_q'(\boldx)\right]\mathbbE_{q(\boldx)}\left[ r_{n_q}(\boldx; \bolddelta) \boldf_q'(\boldx)\right]^\top.\]
By using the newly defined feature function $\boldf_q'$, we are able to write the Hessian in the form of a sample covariance matrix of $\boldf_q'$ weighted by the ratio model $r_{n_q}$. The convergence of such a density ratio reweighted sample covariance matrix has been studied for classic KLIEP algorithm \citep{kim2019}. We state the following proposition: 
\begin{prop}{(Lemma 4 and 5 in \citep{kim2019})}
    Suppose $M^{-1} \le r(\boldx;\bolddelta)\le M$, then $\forall \bolddelta \in \Delta_n, \mathbb{P}\left(\|\nabla^2_\bolddelta \ell(\bolddelta) - \boldSigma \|_\infty \le \epsilon \right) \le C\cdot \exp(-C'\epsilon^2 n)$, where $C, C' >0$ are constants and \[\boldSigma :=  \mathbbE_{q(\boldx)}\left[ r(\boldx; \bolddelta) \boldf_q'(\boldx) \boldf_q'(\boldx)^\top\right] - \mathbbE_{q(\boldx)}\left[ r(\boldx; \bolddelta) \boldf_q'(\boldx)\right]\mathbbE_{q(\boldx)}\left[ r(\boldx; \bolddelta) \boldf_q'(\boldx)\right]^\top.\]
\end{prop}
Notice that $\boldSigma$ is the Hessian of the population objective \eqref{eq.obj}. Considering $\Delta_n \subset \mathbbR^d$ is a compact Euclidean space, we can see that $\sup_{\bolddelta \in \Delta_{n}} \|\boldSigma - \nabla_\bolddelta^2\ell(\bolddelta)\| = o_p(1)$. Combining this with \eqref{eq.lower.hessian}, we can see that $\inf_{\bolddelta\in \Delta_{n}} \infHessLL 
\ge \lambda_\mathrm{min} (\boldSigma) - o_p(1)$. Therefore, our analysis shows the gap of the lowest eigenvalue between  $\nabla_\bolddelta^2\ell(\bolddelta)$ and $\boldSigma$ should be shrinking as $n_q \to \infty$.  Therefore, as long as $\boldSigma$ is bounded away from 0, when $n_q$ is large enough, the $\lambda_\mathrm{min}(\boldSigma)$ is also bounded away from 0.
This claim will be verified later via numerical experiments in supplementary materials, Section \ref{sec.num.consistent}.

\subsection{Convergence of $\nabla_\bolddelta \ell(\bolddelta^*)$}
Now we analyze another condition required for the consistency: \eqref{eq.optimgrad.converge} states $\nabla_\bolddelta \ell(\bolddelta^*)$ should converge to zero in terms of $\ell_2$ norm at the rate $\frac{1}{\sqrt{n_p\wedge n_q}}$. This is guaranteed with high probability:
\begin{prop}
    Suppose $\exists M>0, M^{-1} \le r(\boldx;\bolddelta) \le M$, $\exists K<\infty, \|\mathbbE_{p(\boldx)} \left[\boldf_p'(\boldx)\right]\| \le K$,  $\exists L>0, L^{-1} \le l_p(\boldx) \le L,  \forall \boldx$. Then $\exists C,C'>0, \mathbb{P}\left(\|\nabla_\bolddelta \ell(\bolddelta^*)\|_\infty \ge \epsilon \right) \le C\exp\left[-C'\epsilon^2( n_p \wedge n_q)\right]$.
\end{prop}
\begin{proof}
It can be seen that 
\begin{align*}
    -\nabla_\bolddelta\ell(\bolddelta^*) 
    = \underbrace{\frac{p(\boldy_p)}{\hat{p}(\boldy_p)}\hat{\mathbbE}_{p(\boldx)}\left[ \boldf'_p(\boldx)\right] - \hat{\mathbbE}_{p(\boldx)}\left[ \boldf'_p(\boldx)\right]}_{\bolda} + \underbrace{\hat{\mathbbE}_{p(\boldx)}\left[ \boldf'_p(\boldx)\right] - \hat{\mathbbE}_{q(\boldx)}\left[ r_{n_q}(\boldx;\bolddelta^*) \boldf'_q(\boldx)\right]}_{\boldb}. 
\end{align*}
Moreover, we can see $\|\bolda\| \le \frac{\left|p(\boldy_p) -\hat{p}(\boldy_p)\right|}{p(\boldy_p) \cdot \hat{p}(\boldy_p)}\cdot \|\mathbbE_{p(\boldx)} \left[\boldf_p'(\boldx)\right]\|$ due to Holder's inequality, thus knowing $l_p$ are bounded, Hoeffding inequality shows $\mathbb{P}\left[\|\bolda\|_\infty \ge \epsilon \right] \le C_1'\exp(-C_1''\epsilon n_p)$ where $C_1', C_1'' >0$ are constants.  

It can be seen that $\boldb$ is the difference between a simple sample average over $p$ and another sample average over $q$ reweighted by a ratio model $r_{n_q}$. The convergence of such a term has also been studied in the classic KLIEP algorithm. Here we directly invoke Lemma C.2 in \citep{kim2019}: 
$\mathbb{P}\left[\|\boldb\|_\infty \ge \epsilon \right] \le C_2'\exp(-C_2''\epsilon^2 \min(n_p,n_q))$ where $C_2', C''_2>0$ are constants. .

Combining the boundedness of $\bolda$ and $\boldb$ using triangle inequality yields desired results. 
\end{proof}

\subsection{Asymptotic Normality of $\hat{\bolddelta}$}
\label{sec.ad}
Given the consistency results above, we now establish the normality of the estimator \eqref{eq.emp.obj}.
Let us define 
    $\boldSigma_p := \mathrm{Var}_{p(\boldx)} \left[\boldf_p'(\boldx)\right]$, $ \boldSigma_q := \mathrm{Var}_{q(\boldx)} \left[r(\boldx;\bolddelta^*)\boldf'_q(\boldx)\right],
$
where $\mathrm{Var}_{p(\boldx)}\left[\boldg(\boldx)\right]$ is the covariance operator of a random vector $\boldg(\boldx)$ with respect to $p(\boldx)$.

\begin{thm}
\label{thm.asy}
If $\hat{\bolddelta} \overset{\mathbb{P}}{\to} \bolddelta^*$, 
$
    \sup_{\bolddelta\in \Delta_{n}} \|\nabla_\bolddelta^2 \ell(\bolddelta) - \boldSigma \| \overset{\mathbb{P}}{\to} 0,
$
then as $n_p \to \infty$, $n_q \to \infty$ and $0<\frac{n_p}{n_q} <\infty$,
\begin{align}
\label{eq.asym.norm.thm}
    \sqrt{n_p}\left[\bolddelta^* - \hat{\bolddelta}\right]
    \rightsquigarrow &\mathcal{N}\left[\boldzero, \boldSigma^{-1}\left(\boldSigma_{p} + \frac{n_p}{n_q}\cdot \boldSigma_{q}\right)\boldSigma^{-1}\right]. 
\end{align}
\end{thm}
The proof can be found in supplementary materials. The conditions under which the uniform convergence holds have been studied in Section \ref{sec.eig.hessian}. 
In practice, we can use $\hat{\bolddelta}$ to replace ${\bolddelta}^*$ and use sample average to replace expectations when calculating the above asymptotic covariance.

\section{Lagrangian Dual  of \eqref{eq.emp.obj}}
In some cases, the direct computation of \eqref{eq.emp.obj} is infeasible: For example, when the feature function $\boldf$ is in an RKHS \citep{Scholkopf2001}, $\boldf(\boldx)$ may not be evaluated explicitly. Therefore, we study an alternative to \eqref{eq.emp.obj}.
Denote $\boldf'_{p_n} := \frac{l_p}{\phatp}$. 
We can consider the following bi-level optimization problem:
\begin{align*}
\textstyle
\hat{\bolddelta}_{\mathrm{dual}} := \argmin_{\bolddelta} E_{q_n}(\bolddelta) = \argmin_{\bolddelta} \|\hat{\mathbbE}_{q(\boldx)} \left[r_{n_q}(\boldx;\bolddelta)\boldf'_q(\boldx)\right]- \sum_{i=1}^{n_q} \hat{\mu}_i  \boldf(\boldx_q^{(i)})\|,
\end{align*}
\vspace*{-5mm}
\begin{align}
\label{eq.dual.2}
\textstyle
\hat{\boldmu}&:=\argmin_{\substack{ \boldmu\ge 0, \\
\sum_i \mu_i = 1}}  \sum_{i=1}^{n_q} \mu_i \log \mu_i - \mu_i \log l_q(\boldx_q^{(i)})
\text{ s.t. } \|\hat{\mathbbE}_{p(\boldx)} \left[\boldf'_{p_n}(\boldx)\right] - \sum_{i=1}^{n_q} \mu_i  \boldf(\boldx_q^{(i)}) \| \le R_n,
\end{align}
and $R_n$ is a constant dependent on $n_p \wedge n_q$. If $R_n = 0$, \eqref{eq.dual.2} is the \emph{Lagrangian dual} of $\eqref{eq.emp.obj}$. The proof can be found in the supplementary material. We show $\hat{\bolddelta}_\mathrm{dual}$ is also consistent.
\begin{thm}
	\label{thm.dual.consistency}
	Suppose Assumption \ref{ass.correctmodel} holds, $E_{q_n}(\hat{\bolddelta}_{\mathrm{dual}})\le C_0$, 
	\eqref{eq.optimgrad.converge} holds with probability at least $1-\epsilon_{R_1}$ and $\exists R_2>0, R_3\ge0, C_r > 1$,
	\begin{align}
	\label{eq.mineig.2}
	&\inf_{\bolddelta\in \Delta_{n}}\lambda_{\min}\left[\nabla^2_\bolddelta \ell(\bolddelta)\right]  > R_2,\Delta_{n}:= \mathrm{Ball}\left(\bolddelta^*,\frac{R_1+R_3}{R_2}\frac{C_r}{{\sqrt{n_p\wedge n_q}}} + \frac{C_0}{R_2}\right),
	\end{align} 
	with probability at least $\epsilon_{R_2}$,
	then $\exists R_n \le \frac{R_3}{\sqrt{n_p\wedge n_q}}$ such that $\left\|\bolddelta^* - \hat{\bolddelta}_{\mathrm{dual}}  \right\| \le \frac{R_1+R_3}{R_2}\frac{1}{\sqrt{n_p \wedge n_q}} + \frac{C_0}{{R_2}}$ with probability at least $ 1 -\epsilon_{R_1} - \epsilon_{R_2}$. 
\end{thm}
The proof can be found in Section \ref{sec.proof.thm3} in the supplementary material. In our experiments, we find $E_{q_n}(\hat{\bolddelta}_{\mathrm{dual}})$ is usually around $10^{-7}$, close to the the tolerance of optimization software. Therefore, in practice, we can regard $C_0\approx 0$. 
Moreover, in this formulation,  $\boldf$ only occurs in the inner product thus the evaluation of $\eqref{eq.dual.2}$ can be evaluated using kernel trick when $\boldf$ is in RKHS.  However, an RKHS version of PRE will be a future work. 

\begin{wrapfigure}{r}{.33\textwidth}
    \vspace*{-17mm}
    \centering
    \includegraphics[width=.33\textwidth]{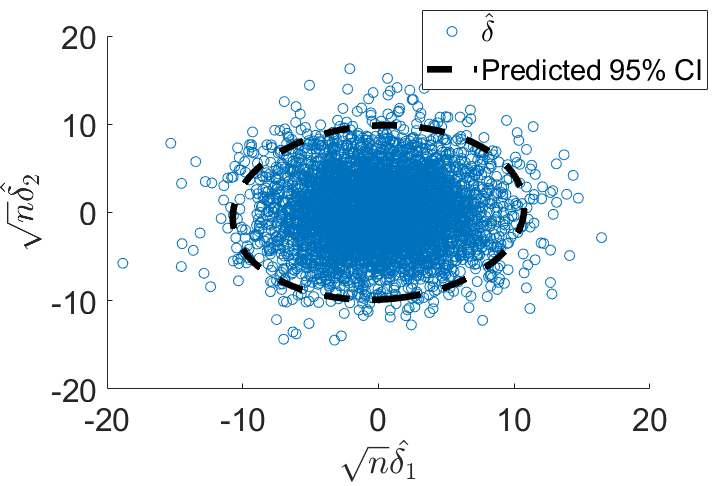}
    \includegraphics[width=.33\textwidth]{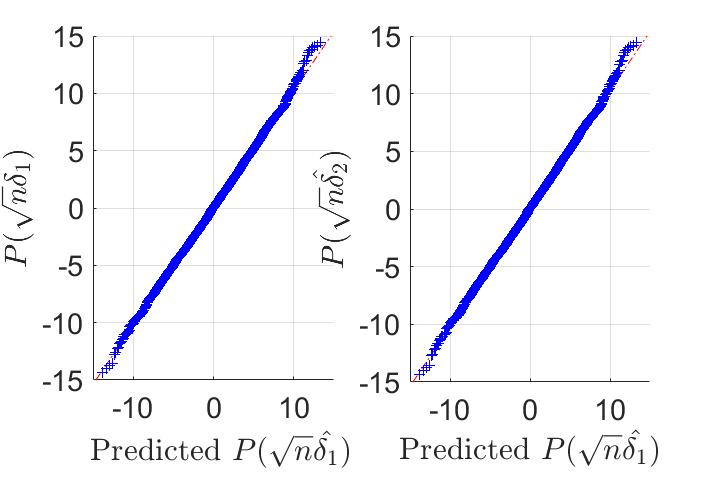}
    \caption{
    Simulation vs. Prediction.
    }
    \label{fig:qqplot}
    \vspace*{-7mm}
\end{wrapfigure}
\section{Numerical Experiments and Applications}
\subsection{Asymptotic Normality of $\hat{\bolddelta}$}
In this experiment, we validate the asymptotic distribution discussed in Section \ref{sec.ad} under a finite-sample setting. 

We generate two sets of 100 i.i.d. samples from  $\mathcal{N}(0.5,0.1^2)$ as $\boldy_p\in \mathbbR^{100}$ and $\boldy_q\in \mathbbR^{100}$ and set the likelihood function $l_p(\boldx) = l_q(\boldx) := \mathcal{N}([x_1\cdot \boldone, x_2 \cdot \boldone], 10^2\cdot \boldI),  \boldx \in \mathbbR^2$, where $\boldone$ is a 50-dimensional vector filled with ones and $\boldI$ is an identity matrix. 500 prior samples $X_p$ and $X_q$ are drawn from $p(\boldx) = q(\boldx) = \mathcal{N}([0,0],\boldI)$. $n_p=n_q=n$, $\boldf(\boldx) = \boldx$. Then
$\hat{\delta}$ is calculated using \eqref{eq.emp.obj} given $X_p, X_q, \boldy_p, \boldy_q, l_p$ and $l_q$. 
We run the estimation of $\hat{\delta}$ 5000 times with different random seeds and scatter-plot ${\sqrt{n} \hat{\delta}}$ in Figure \ref{fig:qqplot}. We also plot  $95\%$ confidence interval (CI) predicted by Theorem \ref{thm.asy} in black on the same plot. It can be seen that the overall distribution of $\sqrt{n}\hat{\bolddelta}$ fits the shape of CI well. 249 (4.98\%) simulations are outside of the confidence interval. 
We also compare the empirical distribution of $\hat{\delta}_1$ and $\hat{\delta}_2$ with 
their theoretical predictions on the quantile quantile plots (qq-plot) in Figure \ref{fig:qqplot}. 

Strictly speaking, in this experiment, $p(\boldx|\boldy_p) \neq q(\boldx|\boldy_q)$ since $\boldy_p \neq \boldy_q$. However, as we take 100 samples from the same distribution as $\boldy_p$ and $\boldy_q$, given this strong evidence, $\ell_p(\boldz) \approx \ell_q(\boldz)$. Thus $\bolddelta^* \approx \boldzero$. 
Both plots show that the prediction made by Theorem \ref{thm.asy} matches the simulations well.

\subsection{Application: Latent Signal Detection}
\label{sec.latent.detection}
Given a reference time-sequence $\boldy_p\in\mathbbR^d$ that is drawn from a ``background'' distribution $p(\boldy)$, we would like to know whether a testing sequence $\boldy_q\in\mathbbR^d$ is standing out from the background or not.
Moreover, 
$\boldy$ is usually a noisy version of a \emph{latent signal} $\boldx$ and we are more interested in detecting the latent signal. 
In change detection literature, a common strategy is checking if these two sequences are significantly different in terms of some \emph{distance measure}. 

For example, let $\boldy_p$ or $\boldy_q$ be the vibration of a car engine at a certain time. Though it is difficult to study $\boldy_p$ or $\boldy_q$ directly, 
its latent vibration pattern $\boldx$ can be understood by analyzing the mechanics of the engine. We can prepare simulations in two possible scenarios: ``normal state'', $p(\boldx)$ and ``damage state'', $q(\boldx)$. Our assumption is, if $\boldy_q$ is not a background signal, the posterior ratio $p(\boldx|\boldy_p)/q(\boldx|\boldy_
q)$ must deviate from 1. Therefore, we propose to use $\ell(\hat{\bolddelta})$ as the detection distance measure.

We first demonstrate this application using a numerical experiment. We introduce a simple random process $x_t = \alpha  x_{t-1} + \epsilon, \epsilon \sim \mathcal{N}(0,.1)$, $y_t = x_{t} + \epsilon', \epsilon \sim \mathcal{N}(0,.02)$. $\boldy_p \in \mathbbR^{50}$ is obtained by running this process fixing $\alpha = 0.5$ for 100 time steps then take $\boldy_p = [y_{51} \dots y_{100}]$. We generate $\boldy_q \in \mathbbR^{50}$ in the same way by fixing $\alpha = -0.2$. 
100,000 sequences from $p(\boldx)$
are simulated using the above process but with a random draw of $\alpha \sim \mathcal{N}(0.5,0.1^2)$. 100,000 
sequences from $q(\boldx)$ are generated using $\alpha \sim \mathcal{N}(0,.5^2)$. It means that we anticipate the latent signal $p$ will have a different auto-correlation coefficient compared to the background $q$. Note since we introduced an additional prior on $\alpha$, both $p(\boldx)$ and $q(\boldx)$ do not have tractable expressions. 

We generate 100 background  $\boldy_p$ and signal $\boldy_q$, then use $\ell(\hat{\bolddelta})$ for signal detection. 
$f(\boldx) = \mathrm{AR}(\boldx,20)$, which is the auto-correlation coefficient calculated from $\boldx$.
Considering our time-sequences are generated by state-space models,  we also compare with several other distance measures: \emph{auto-corr}: $\ell_2$ norm of the difference between AR(20) model parameters fitted on $\boldy_p$ and $\boldy_q$ by MLE, \emph{State-Space Transformation  (SST) \citep{Moskvina2003}}: a method specifically for detecting changes of state-space models.
The ROC curves of different distance measures are shown in Figure \ref{fig:latentchange}. The result shows, in terms of AUC, the PRE metric has a significant lead among all methods. 

\begin{figure}
    \centering
    \includegraphics[width=.3\textwidth]{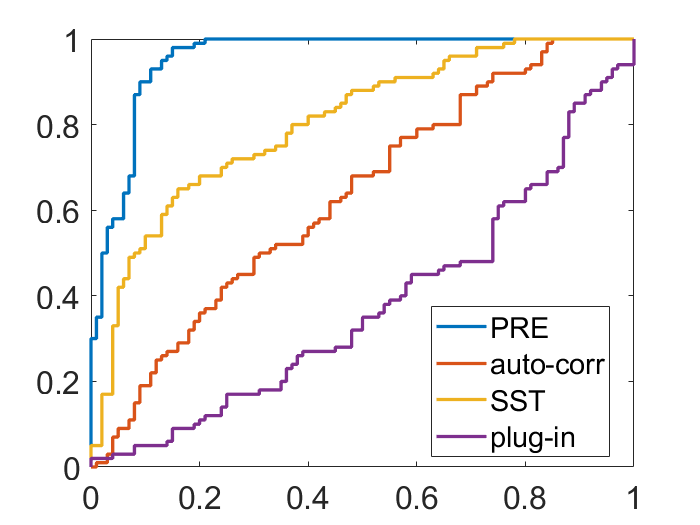}
    \includegraphics[width=.3\textwidth]{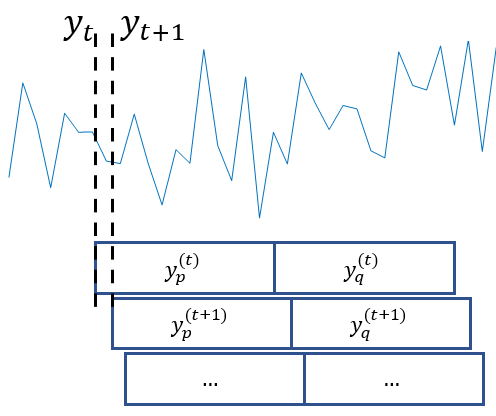}
    \includegraphics[width=.35\textwidth]{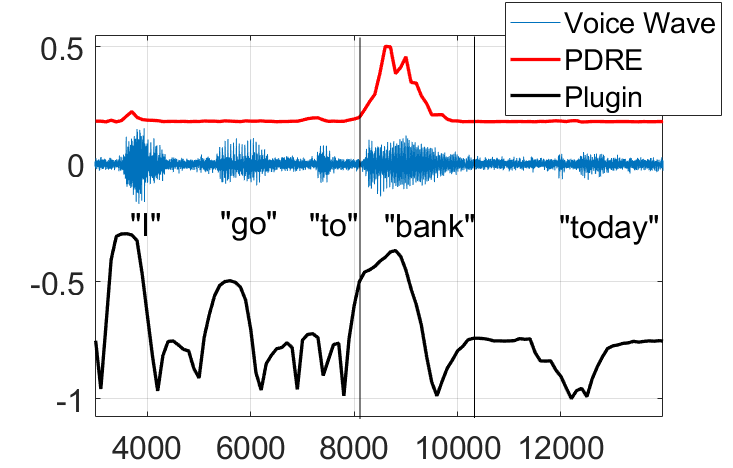}
    \caption{Left: ROC curves of latent signals detection from artificial time-sequences. Center: Illustration of a sliding window based time-series change detection. Right: Speech keyword detection. }
    \label{fig:latentchange}
\end{figure} 

Next, we show how this signal detection technique can be used in a time-series event detection problem. Suppose we have a time-series $y_t, t= 1\cdots T$. We can create $\boldy_p^{(t)}$ and $\boldy_q^{(t)}$ using two consecutive time windows (as shown in Figure \ref{fig:latentchange}). Using the above technique we can compute a metric $\ell_t(\hat{\boldtheta})$. By sliding two windows forward, we can compute $\ell_{t+1}$, $\ell_{t+2}$, etc. Now we can use $\ell_t$ as a running event detection score, to detect possible signals in time-series. We record a speech audio time-series in a noisy background, which says ``I go to bank today.'' We also record the speaker repeating the word ``bank'' in a quiet background for 10 times, then cut the the second speech into small sliding windows sized 500 to create $\boldx_p \in \mathbbR^{500} \sim p(\boldx)$. We let $\boldx_q \in \mathbbR^{500} \sim \mathcal{N}(0,0.01^2)$. The detection result is plotted in Figure \ref{fig:latentchange} where we can see the event-detection score clearly peaks at the word ``bank''. In comparison, the plugin estimator gives a much less clear detection score and does not seem to peak at the desired signal word ``bank''.

\subsection{Application: Extracting Locally Linear Logistic Classifier}
\label{sec.locallinear}
\begin{figure}
    \centering 
    \includegraphics[width=.32\textwidth]{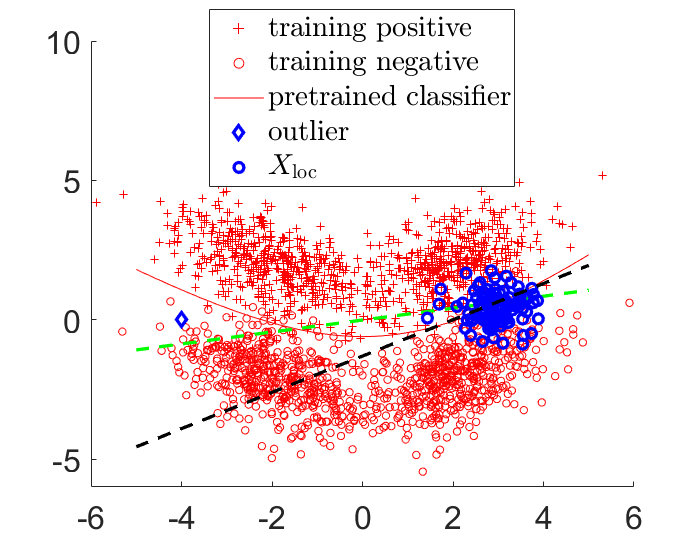}
    \includegraphics[width=.32\textwidth]{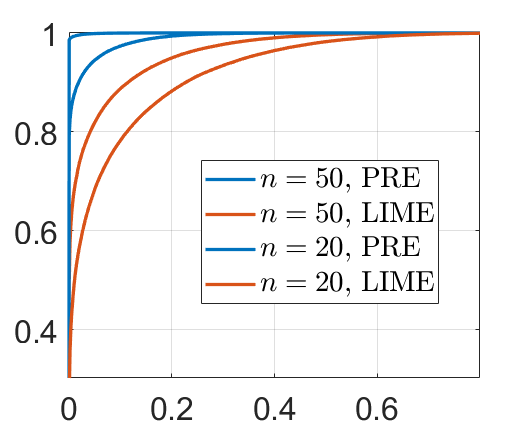} \includegraphics[width=.26\textwidth]{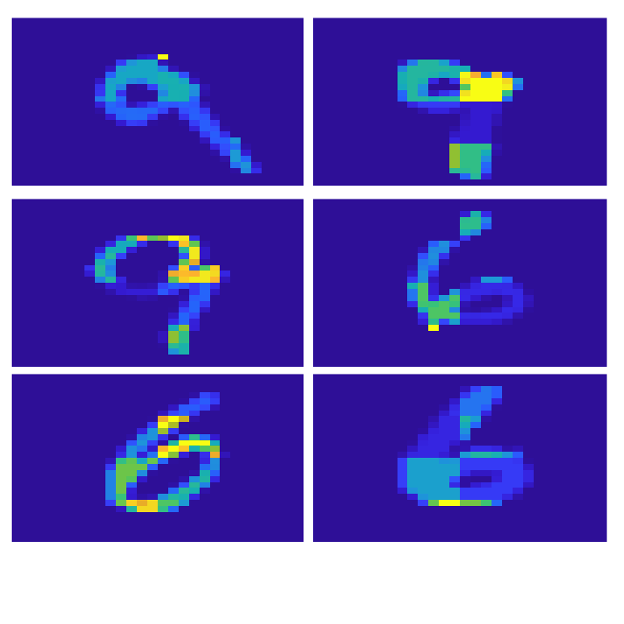}
    \caption{Left: Linear classification boundary extracted by PRE and LIME. Center: ROC curve of extracted classifiers on  $X_{\mathrm{loc}}$, Right: Interpretable features 
    identified by PRE for a few random $\boldx$}.
    \label{fig:localclassifier}
\end{figure} 

PRE can also be used to extract a locally linear classifier from a black-box non-linear probabilistic classifier. Formally speaking, given a black-box, non-linear classifier $p(y|\boldx)$, we would like to obtain a linear logistic classifier, i.e., $p(y|\boldx; \bolddelta, \delta_0) := \frac{1}{1+\exp(\bolddelta^\top\boldx + \delta_0)}$, such that within a local neighbourhood specified by a user, the linear classifier works similarly comparing to the original black-box classifier. This technique is crucial when the user demands the interpretability of a black-box decision. For simplicity, we focus on a binary classification problem in this section, i.e., $y\in \{-1,1\}$. 

Using Bayes rule and some algebra,  we can see $p(y|\boldx)= \frac{1}{1+\frac{p(\boldx|-1)}{p(\boldx|1)}\cdot c}, $ where $c$ is a constant that does not depend on $\boldx$. User usually specifies a local region by constructing a dataset $X_\mathrm{loc} = \{\boldx^{(i)}\}_{i=1}^n \sim p_{\mathrm{loc}}(\boldx)$ which contains perturbations of a data-point $\boldx_0$. 
Thus, extracting parameters of a linear logistic classifier becomes fitting the  $\log\frac{p(\boldx|-1)}{p(\boldx|1)}$ using a linear model, given $p(y|\boldx)$ and $X\sim p_\mathrm{loc}(\boldx)$ as inputs. This can be efficiently solved using PRE by letting $p(\boldy_p|\boldx) = p(-|\boldx)$, $q(\boldy_p|\boldx) = p(+|\boldx)$, $p(\boldx) = q(\boldx) = p_\mathrm{loc}(\boldx)$. Although estimating $\frac{p(\boldx|-1)}{p(\boldx|1)}$ using PRE seems to require the samples from the marginal $p(\boldx)$, not $p_{\mathrm{loc}}(\boldx)$, as $\frac{p(\boldx|-1)}{p(\boldx|1)} \propto \frac{p(-1|\boldx)p(\boldx)}{p(1|\boldx)p(\boldx)} = \frac{p(-1|\boldx)p_\mathrm{loc}(\boldx)}{p(1|\boldx)p_{\mathrm{loc}}(\boldx)}$, using samples from $p_{\mathrm{loc}}(\boldx)$ in PRE would also yield a consistent estimator of $\frac{p(\boldx|-1)}{p(\boldx|1)}$. 

Traditionally, such a local linear classifier is obtained by minimizing the discrepancy of the predictions between the pre-trained classifier $p(y|\boldx)$ and the local classifier $p(y|\boldx; \bolddelta, \delta_0)$ over $X_\mathrm{loc}$ (See, e.g., LIME, \citep{ribeiro2016should}). 
However, we show that such a process is not necessarily robust. See Figure \ref{fig:localclassifier} for an example: An outlier (blue diamond) is included in $X_{\mathrm{loc}}$. 
Linear classifier obtained by PRE (black dotted line) nicely recovers an accurate linear boundary around the majority samples of $X_\mathrm{loc}$. 
However, as the pretrained classifier labels the outlier as a negative example, LIME (green dotted line) tries to mimic such a prediction, resulting in a skewed classification boundary. 
Intuitively, this issue can be understood as that LIME estimates a sigmoid transform of the class density ratio rather than the ratio itself.
This may magnify the effect of outlier samples.

The extracted local classifier must also perform similarly to the original blackbox classifier on $X_\mathrm{loc}$. 
Figure \ref{fig:localclassifier} shows the ROC curves computed on the same dataset for two extracted classifiers. The ``true'' label in this case is defined as the predicted label by the black-box classifier. PRE classifier (red) achieves a significant lead comparing to LIME (blue) when $n = 20$ and $40$. 
Finally, we train a two layer neural network binary calssifier on MNIST dataset using digits ``6'' and ``9'' then extract a linear classifier. We generate $X_\mathrm{loc}$ by first picking an arbitrary image, then  randomly setting some of its super-pixels' value to 0. One image is divided into 7 $\times$ 7 super-pixels. The probability of each super-pixel turning black is a random coin flip. In this way, we obtain $X_\mathrm{loc}$ of size 1000. We use $\boldf(\boldx) = \boldx$ for PRE and the magnitude of $\hat{\bolddelta}$ for each picture is plotted in Figure \ref{fig:localclassifier}. It can be seen that the selected pixels (on which $|\delta_i| >0 $) are indeed parts of the written digits thus should be important features of the black-box classifier. 

\section{Conclusions}
We studied the problem of learning latent distribution changes via  posterior ratio estimation. 
We propose a novel method that directly estimates the posterior ratio from two observed datasets,
two likelihood functions and two sets of synthetic samples from priors. The proposed method is consistent and the estimated parameters have a limiting normal distribution. 
We consider two practical applications: Latent signal detection and local linear classifier extraction. Both experiments report promising results. 
\newpage
\bibliography{main}

\begin{thebibliography}{10}

\bibitem{Arjovsky17WGAN}
M.~Arjovsky, S.~Chintala, and L.~Bottou.
\newblock {W}asserstein generative adversarial networks.
\newblock In {\em Proceedings of the 34th International Conference on Machine
  Learning}, volume~70 of {\em Proceedings of Machine Learning Research}, pages
  214--223, 2017.

\bibitem{frogner2015learning}
C.~Frogner, C.~Zhang, H.~Mobahi, M.~Araya, and T.~A. Poggio.
\newblock Learning with a wasserstein loss.
\newblock In {\em Advances in Neural Information Processing Systems}, pages
  2053--2061, 2015.

\bibitem{genevay2016stochastic}
A.~Genevay, M.~Cuturi, G.~Peyr{\'e}, and F.~Bach.
\newblock Stochastic optimization for large-scale optimal transport.
\newblock In {\em Advances in neural information processing systems}, pages
  3440--3448, 2016.

\bibitem{Goodfellow2014}
I.~Goodfellow, J.~Pouget-Abadie, M.~Mirza, B.~Xu, D.~Warde-Farley, S.~Ozair,
  A.~Courville, and Y.~Bengio.
\newblock Generative adversarial nets.
\newblock In {\em Advances in neural information processing systems 27}, pages
  2672--2680, 2014.

\bibitem{gretton2012kernel}
A.~Gretton, K.~M. Borgwardt, M.~J. Rasch, B.~Sch{\"o}lkopf, and A.~Smola.
\newblock A kernel two-sample test.
\newblock {\em Journal of Machine Learning Research}, 13(Mar):723--773, 2012.

\bibitem{Horn1986}
R.~A. Horn and C.~R. Johnson, editors.
\newblock {\em Matrix Analysis}.
\newblock Cambridge University Press, 1986.

\bibitem{kanamori2011f}
T.~Kanamori, T.~Suzuki, and M.~Sugiyama.
\newblock $ f $-divergence estimation and two-sample homogeneity test under
  semiparametric density-ratio models.
\newblock {\em IEEE Transactions on Information Theory}, 58(2):708--720, 2011.

\bibitem{Kawahara2012}
Y.~Kawahara and M.~Sugiyama.
\newblock Sequential change-point detection based on direct density-ratio
  estimation.
\newblock {\em Statistical Analysis and Data Mining}, 5(2):114--127, 2012.

\bibitem{kim2019}
B.~Kim, S.~Liu, and M.~Kolar.
\newblock Two-sample inference for high-dimensional markov networks.
\newblock {\em arXiv:1905.00466}, 2019.

\bibitem{Moskvina2003}
V.~Moskvina and A.~Zhigljavsky.
\newblock Change-point detection algorithm based on the singular-spectrum
  analysis.
\newblock {\em Communications in Statistics: Simulation and Computation},
  32:319--352, 2003.

\bibitem{Nguyen2008}
X.~Nguyen, M.~J. Wainwright, and Michael~I. J.
\newblock Estimating divergence functionals and the likelihood ratio by
  penalized convex risk minimization.
\newblock In {\em Advances in Neural Information Processing Systems 20}, pages
  1089--1096. 2008.

\bibitem{Nowozin2016}
S.~Nowozin, B.~Cseke, and R.~Tomioka.
\newblock f-gan: Training generative neural samplers using variational
  divergence minimization.
\newblock In {\em Advances in Neural Information Processing Systems}, pages
  271--279, 2016.

\bibitem{parzen1962estimation}
E.~Parzen.
\newblock On estimation of a probability density function and mode.
\newblock {\em Ann. Math. Statist.}, 33(3):1065--1076, 1962.

\bibitem{quionero2009dataset}
J.~Quionero-Candela, M.~Sugiyama, A.~Schwaighofer, and N.~D. Lawrence.
\newblock {\em Dataset shift in machine learning}.
\newblock The MIT Press, 2009.

\bibitem{ribeiro2016should}
M.~T. Ribeiro, S.~Singh, and C.~Guestrin.
\newblock "why should i trust you?" explaining the predictions of any
  classifier.
\newblock In {\em Proceedings of the 22nd ACM SIGKDD international conference
  on knowledge discovery and data mining}, pages 1135--1144, 2016.

\bibitem{rosenblatt1956}
M.~Rosenblatt.
\newblock Remarks on some nonparametric estimates of a density function.
\newblock {\em Ann. Math. Statist.}, 27(3):832--837, 09 1956.

\bibitem{Scholkopf2001}
B.~Scholkopf and A.~J. Smola.
\newblock {\em Learning with kernels: support vector machines, regularization,
  optimization, and beyond}.
\newblock MIT press, 2001.

\bibitem{Sugiyama2012}
M.~Sugiyama, T.~Suzuki, and T.~Kanamori.
\newblock {\em Density Ratio Estimation in Machine Learning}.
\newblock Cambridge University Press, 2012.

\bibitem{Sugiyama2008}
M.~Sugiyama, T.~Suzuki, S.~Nakajima, H.~Kashima, P.~von B\"unau, and
  M.~Kawanabe.
\newblock Direct importance estimation for covariate shift adaptation.
\newblock {\em Annals of the Institute of Statistical Mathematics},
  60(4):699--746, 2008.

\bibitem{Tsuboi2009}
Y.~Tsuboi, H.~Kashima, S.~Hido, S.~Bickel, and M.~Sugiyama.
\newblock Direct density ratio estimation for large-scale covariate shift
  adaptation.
\newblock {\em Journal of Information Processing}, 17:138--155, 2009.

\end{thebibliography}
\newpage

\section{The Dual Formulation of \eqref{eq.emp.obj}}

$\min_\bolddelta \ell(\bolddelta)$ can be re-written as 
\begin{align*}
    \min_{\bolddelta,t_i}     & -\hat{\mathbb{E}}_{p(\boldx)} \left[ \frac{l_p(\boldx)}{\phatp} \cdot \langle \bolddelta, \boldf(\boldx) \rangle \right] 
    + \log \hat{\mathbbE}_{q(\boldx)} \left[l_q(\boldx) \exp(t)\right], \\
    &\text{subject to }
    \forall i \in \{1\dots n_q\}, t_i =  \langle \bolddelta, \boldf(\boldx_q^{(i)}) \rangle.
\end{align*}
Write the Lagrangian of the above convex optimization, 
\begin{align*}
    \max_{\mu_i}\min_{\bolddelta,t_i}     & -\hat{\mathbb{E}}_{p(\boldx)} \left[ \frac{l_p(\boldx)}{\hat{p}(\boldy_p)} \cdot \langle \bolddelta, \boldf(\boldx) \rangle \right] 
    + \log \hat{\mathbbE}_{q(\boldx)} \left[l_q(\boldx)\exp(t)\right] - \sum_{i=1}^{n_q} \mu_i \left[t_i -  \langle \bolddelta, \boldf(\boldx_q^{(i)}) \rangle \right],
\end{align*}
where $\mu_i$ are Lagrangian multipliers. First, we solve above minimization with respect to $t_i$ (by setting the derivative to zero), and it can be seen that the optimum is attained when $\mu_i = \frac{l_q\left(\boldx_q^{(i)}\right) \exp(t_i) }{\hat{\mathbbE}_{q_{\boldx}}\left[l_q\left(\boldx\right)\exp(t)\right]}$, $\sum_{i=1}^{n_q} \mu_i = 1, \mu_i \ge 0$. 
Substituting above optimality condition to the objective function, we obtain 
\[\max_{\substack{\mu_i\ge 0, \\\sum_{i=1}^{n_q} \mu_i  = 1 }}\min_{\bolddelta} -\hat{\mathbb{E}}_{p(\boldx)} \left[ \frac{l_p(\boldx)}{\hat{p}(\boldy_p)} \cdot \langle \bolddelta, \boldf(\boldx) \rangle \right] + \sum_{i=1}^{n_q} -\mu_i \log(\mu_i)  +  \mu_i \log l_q(\boldx_q^{(i)}) + \mu_i \langle \bolddelta, \boldf(\boldx_q^{(i)}) \rangle,  \]

Second, we solve the minimization with respect to $\bolddelta$, one can derive the optimality condition $\hat{\mathbb{E}}_{p(\boldx)} \left[ \frac{l_p(\boldx)}{\phatp} \cdot \boldf(\boldx) \right] =  \sum_{i=1}^{n_q} \mu_i  \boldf(\boldx_q^{(i)}) $. Substituting the above optimality condition, we obtain the objective function
\begin{align*}
    \max_{\mu_i\ge 0, \sum_{i=1}^{n_q} \mu_i  = 1 } &\sum_{i=1}^{n_q} -\mu_i \log(\mu_i) +  \mu_i \log l_q(\boldx_q^{(i)}), \\
    \text{subject to: } &
    \hat{\mathbb{E}}_{p(\boldx)} \left[ \frac{l_p(\boldx)}{\phatp} \cdot  \boldf(\boldx) \right] =  \sum_{i=1}^{n_q} \mu_i   \boldf(\boldx_q^{(i)}).
\end{align*}
The above objective function is equivalent to \eqref{eq.dual.2} when $R_n = 0$.
\newpage
\section{Proof of Theorem \ref{them.2}}

\begin{proof}
Let us begin by considering the following constrainted optimization problem:
\begin{align}
    \label{eq.emp.obj.con}
    \check{\bolddelta} := \argmin_{\bolddelta\in \Delta_{n}} \ell(\bolddelta).
\end{align}

The Lagrangian of \eqref{eq.emp.obj.con} is: 
$   \mathrm{Lag} (\mu,\bolddelta) := \ell(\bolddelta) + \mu \cdot \|\bolddelta^* - \bolddelta\| - \mu \cdot \frac{C_r}{{\sqrt{n_p\wedge n_q}}}.
$
The KKT condition states that the minimizer $\check{\bolddelta}$ must satisfy 
$\boldzero = \nabla_\bolddelta \ell(\check{\bolddelta}) + \hat{\mu} \cdot \frac{\bolddelta^* - \check{\bolddelta}}{\|\bolddelta^* - \check{\bolddelta}\|}$ for some $\hat{\mu} \ge 0$. Rewrite this equality by expanding $\nabla_\bolddelta \ell(\check{\bolddelta})$  at $\bolddelta^*$ using mean value theorem in a coordinate-wise fashion: 
\[\boldzero = \nabla_\bolddelta \ell(\bolddelta^*) +\left[\nabla^2_\bolddelta\ell(\bar{\bolddelta}) + \frac{\hat{\mu}}{\|\bolddelta^* - \check{\bolddelta}\|}\boldI \right] \left[\bolddelta^* - \check{\bolddelta}\right],\]
where $\bar{\delta}_i := h\hat{\delta}_i + (1-h)\delta^*_i$ for some $h\in[0,1]$ which implies $(\bar{\delta}_i - \delta_i^*)^2 = h^2(\hat{\delta}_i - \delta_i^*)^2 \le (\hat{\delta}_i - \delta_i^*)^2$. Therefore, $\bar{\bolddelta} \in \Delta_{n}$. 

After some algebra and using the fact that $\bar{\bolddelta} \in \Delta_{n}$ we get 
\begin{align*}
    \|\bolddelta^* - \check{\bolddelta}\| &= \left\|\left[\nabla^2_\bolddelta\ell(\bar{\bolddelta}) + \frac{\hat{\mu}}{\|\bolddelta^* - \check{\bolddelta}\|}\boldI \right]^{-1} \nabla_\bolddelta \ell(\bolddelta^*)\right\| \\
    &< \left\|\left[\nabla^2_\bolddelta\ell(\bar{\bolddelta}) \right]^{-1}\right\|\cdot \|\nabla_\bolddelta \ell(\bolddelta^*)\|\le \frac{1}{R_2}\cdot\|\nabla_\bolddelta \ell(\bolddelta^*)\|
    \le \frac{R_1}{R_2} \cdot \frac{1}{\sqrt{n_p\wedge n_q}}
\end{align*}
where the first inequality is due to Holder's inequality and the second inequality first uses an equality $\|\boldA^{-1}\| = 1/\lambda_\mathrm{min}(A), $ if $\boldA$ is p.s.d., then applies inequality \eqref{eq.mineig}. The last inequality is due to \eqref{eq.optimgrad.converge}. 

\eqref{eq.optimgrad.converge} holds with probability at least 1-$\epsilon_{R_1}$ and \eqref{eq.mineig} holds with probability at least 1-$\epsilon_{R_2}$. A union bound shows the above inequality holds with probability at least $1-\epsilon_{R_1} - \epsilon_{R_2}$. 

Since the ball constraint has radius $\frac{R_1}{R_2} \cdot\frac{C_r}{\sqrt{n_p\wedge n_q}}$ and $C_r>1$, we conclude the minimizer $\check{\bolddelta}$ is in the interior of $\Delta_{n}$ with probability at least $1-\epsilon_{R_1} - \epsilon_{R_2}$.
The slackness condition  states, $\exists \hat{\mu}\ge 0, \hat{\mu}\cdot\left(\|\bolddelta^* - \check{\bolddelta}\| - \frac{R_1}{R_2} \cdot\frac{C_r}{\sqrt{n_p\wedge n_q}}\right) = 0$. Knowing $\check{\bolddelta}$ is in the interior of $\Delta_{n}$, we can deduce $\hat{\mu} = 0$, which means the constraint $\bolddelta \in \Delta_{n}$ is inactive, so $\hat{\bolddelta}=\check{\bolddelta}$ 
with probability at least $1-\epsilon_{R_1} - \epsilon_{R_2}$.

The fact that $\check{\bolddelta}$ is in the interior of a $\bolddelta^*$-centered ball whose radius shrinks at the rate $\frac{1}{\sqrt{n_p\wedge n_q}}$ allows us to conclude that $\check{\bolddelta}  \overset{\mathbb{P}}{\to} \bolddelta^*$. Finally $\hat{\bolddelta}  \overset{\mathbb{P}}{\to} \bolddelta^*$ due to $\hat{\bolddelta} = \check{\bolddelta}$ with probability at least $1-\epsilon_{R_1} - \epsilon_{R_2}$ as discussed above. 
\end{proof}
\newpage
\section{Proof of Theorem \ref{thm.asy}}
\begin{proof}
This theorem is the consequence of the estimation equation of \eqref{eq.emp.obj}:
\[\boldzero = \nabla_\bolddelta \ell(\hat{\bolddelta}) = \nabla_\bolddelta \ell(\bolddelta^*) +\nabla^2_\bolddelta\ell(\bar{\bolddelta}) \left[\bolddelta^* - \hat{\bolddelta}\right].\]
where the rightmost equality is due to row-wise Mean Value Theorem.
Thus, 
\begin{align*}
    -\nabla_\bolddelta \ell(\bolddelta^*) &= \nabla^2_\bolddelta\ell(\bar{\bolddelta})  \left[\bolddelta^* - \hat{\bolddelta}\right] 
    =\left[\boldSigma + o_p(1)\right]{\left[\bolddelta^* - \hat{\bolddelta}\right]}, 
\end{align*}
where the last equality is due to the uniform convergence of the Hessian assumed in our theorem and Continuous Mapping Theorem.

Therefore 
\begin{align}
    \label{eq.asy.norm}
    \left[\hat{\bolddelta} - \bolddelta^*\right]=\left[\boldSigma + o_p(1)\right]^{-1}\nabla_\bolddelta \ell(\bolddelta^*).
\end{align}

Expanding  $\nabla_\bolddelta\ell(\bolddelta^*)$ and multiply $\sqrt{n_q}$, 
\begin{align}
\label{eq.normality}
    -\sqrt{n_p}\nabla_\bolddelta\ell(\bolddelta^*)  = &  
     \sqrt{n_p}\left\{ \left\{1 + \frac{\left[p(\boldy_p) - \phatp\right]}{\phatp}\right\}\hat{\mathbbE}_{p(\boldx)} \left[\boldf'_p(\boldx)\right] - \mathbbE_{p(\boldx)} \left[\boldf'_p(\boldx)\right]\right\} \notag \\
     +& 
     \sqrt{n_q}\cdot \frac{\sqrt{n_p}}{\sqrt{n_q}}\left\{\mathbbE_{q(\boldx)}\left[r(\boldx;\bolddelta^*) \boldf'_q(\boldx)\right] 
    - \hat{\mathbbE}_{q(\boldx)}\left[r(\boldx;\bolddelta^*)\boldf'_q(\boldx)\right]\right\}\notag \\
    \overset{\mathbb{P}}{\to}&     \sqrt{n_p}\left\{ \hat{\mathbbE}_{p(\boldx)}  \left[\underbrace{\boldf'_p(\boldx) - \mathbbE_{p(\boldx)} \left[\boldf'_p(\boldx)\right]}_{\bolda} \right] \right\} \notag \\
     -& 
     \sqrt{n_q}\cdot \frac{\sqrt{n_p}}{\sqrt{n_q}}\left\{ 
    \hat{\mathbbE}_{q(\boldx)}\left[\underbrace{r(\boldx;\bolddelta^*)\boldf'_q(\boldx)- \mathbbE_{q(\boldx)}\left[r(\boldx;\bolddelta^*) \boldf'_q(\boldx)\right]}_{\boldb}\right]\right\}. 
\end{align}

The first equality is due to $\mathbbE_{p(\boldx)} \left[\boldf'_p(\boldx)\right] = \mathbbE_{q(\boldx)}\left[r(\boldx;\bolddelta^*) \boldf'_q(\boldx)\right]$, for the optimal parameter $\bolddelta^*$.

Since $\boldy_p, \boldy_q$ and $\boldx_p, \boldx_q$ are independent of each other, $\bolda$ and $ \boldb$ are two \emph{independent} zero-mean random variables with covariance $\boldSigma_p$ and $\frac{n_p}{n_q}\boldSigma_q $ respectively. 

Applying CLT on \eqref{eq.normality} yields $-\sqrt{n_p}\nabla_\bolddelta\ell(\bolddelta^*)\rightsquigarrow \mathcal{N}(\boldzero, \boldSigma_p + \frac{n_p}{n_q}\boldSigma_q)$. Finally, \eqref{eq.asy.norm} indicates that we should left and right multiply the covariance by $\boldSigma^{-1}$ to obtain covariance of $\sqrt{n_p}\cdot (\bolddelta^* - \hat{\bolddelta})$. This yields desired results in \eqref{eq.asym.norm.thm}.
\end{proof}

\newpage

\section{Proof of Theorem \ref{thm.dual.consistency}}
\label{sec.proof.thm3}
Similar to Theorem \ref{them.2}, we consider optimizing $\bolddelta$ with a constraint $\bolddelta \in \Delta_{n}$. 

First, due to the optimization constraint in \eqref{eq.dual.2} we know:
\begin{align*}
    \frac{R_3}{\sqrt{n_p \wedge n_q}} \ge R_n\ge&\left\|\hat{\mathbbE}_{p(\boldx)} \left[\boldf_{p_n}'(\boldx)\right] - \sum_{i=1}^{n_q} \hat{\mu}_i  \boldf(\boldx^{(i)}_q) \right\|\\
    =& \left\|\hat{\mathbbE}_{p(\boldx)} \left[\boldf_{p_n}'(\boldx)\right] -\hat{\mathbbE}_{q(\boldx)} \left[ r^*_{n}\boldf'_q(\boldx)\right] +  \hat{\mathbbE}_{q(\boldx)} \left[ r^*_{n}\boldf'_q(\boldx)\right] - \sum_{i=1}^{n_q} \hat{\mu}_i  \boldf(\boldx_q^{(i)}) \right\|\\
    \ge& -\left\|\hat{\mathbbE}_{p(\boldx)} \left[\boldf_{p_n}'(\boldx)\right] - \hat{\mathbbE}_{q(\boldx)} \left[r^*_{n}\boldf'_q(\boldx)\right] \right\| + \left\| \hat{\mathbbE}_{q(\boldx)} \left[r^*_{n}\boldf'_q(\boldx)\right] - \sum_{i=1}^{n_q} \hat{\mu}_i  \boldf(\boldx_q^{(i)}) \right\|\\
    =& -\left\|\nabla_\bolddelta \ell(\bolddelta^*)\right\| +\left\| \hat{\mathbbE}_{q(\boldx)} \left[r^*_{n}\boldf'_q(\boldx)\right] - \sum_{i=1}^{n_q} \hat{\mu}_i  \boldf(\boldx_q^{(i)})\right\|, 
\end{align*}
where $r_{n_q}(\boldx;\bolddelta^*)$ is shortened as $r^*_n$.
As $\|\nabla_\bolddelta \ell(\bolddelta^*)\| \le \frac{R_1}{\sqrt{n_p \wedge n_q}}$ is assumed with probability at least $1 - \epsilon_{R_1}$, we can see 
\begin{align}
\label{eq.dual.half}
    \left\|\hat{\mathbbE}_{q(\boldx)} \left[ r_{n}^*\boldf_{q}'(\boldx)\right] - \sum_{i=1}^{n_q} \hat{\mu}_i  \boldf(\boldx_q^{(i)})\right\| \le \frac{R_1+R_3}{\sqrt{n_p \wedge n_q}}, \text{  with probability at least $1 - \epsilon_{R_1}$}.
\end{align}
Moreover, let $\hat{r}_{n_q} := r_{n_q}(\boldx;\hat{\bolddelta}_{\mathrm{dual}})$, 
\begin{align*}
    &\left\|\hat{\mathbbE}_{q(\boldx)} \left[ r^*_{n}\boldf'_q(\boldx)\right] - \sum_{i=1}^{n_q} \hat{\mu}_i  \boldf(\boldx_q^{(i)})\right\| \\
    =& 
    \left\|\hat{\mathbbE}_{q(\boldx)} \left[ r^*_{n}\boldf'_q(\boldx)\right] - \hat{\mathbbE}_{q(\boldx)} \left[ \hat{r}_{n_q}\boldf'_q(\boldx)\right] + \hat{\mathbbE}_{q(\boldx)} \left[\hat{r}_{n_q}\boldf'_q(\boldx)\right]- \sum_{i=1}^{n_q} \hat{\mu}_i  \boldf(\boldx_q^{(i)})\right\|\\
    \ge &\left\| \hat{\mathbbE}_{q(\boldx)} \left[r^*_{n}\boldf'_q(\boldx)\right] -  \hat{\mathbbE}_{q(\boldx)} \left[\hat{r}_{n_q}\boldf'_q(\boldx)\right] \right\| - \left\| \hat{\mathbbE}_{q(\boldx)} \left[\hat{r}_{n_q}\boldf'_q(\boldx)\right]- \sum_{i=1}^{n_q} \hat{\mu}_i  \boldf(\boldx_q^{(i)})\right\|\\
    = &\left\|\hat{\mathbbE}_{q(\boldx)} \left[r^*_{n}\boldf'_q(\boldx)\right] - \hat{\mathbbE}_{q(\boldx)} \left[\hat{r}_{n_q}\boldf'_q(\boldx)\right] \right\| - E_{q_n}(\hat{\bolddelta}_{\mathrm{dual}})  \\
    = &\left\|\hat{\mathbbE}_{q(\boldx)} \left[ r^*_{n}\boldf'_q(\boldx)\right] - \hat{\mathbbE}_{p(\boldx)} \left[\boldf_{p_n}'(\boldx)\right] + \hat{\mathbbE}_{p(\boldx)} \left[\boldf_{p_n}'(\boldx)\right] - \hat{\mathbbE}_{q(\boldx)} \left[\hat{r}_{n_q}\boldf'_q(\boldx)\right] \right\| - E_{q_n}(\hat{\bolddelta}_{\mathrm{dual}})  \\
    = &\left\| \nabla_\bolddelta \ell(\bolddelta^*) - \nabla_\bolddelta \ell(\hat{\bolddelta}_{\mathrm{dual}})  \right\| - E_{q_n}(\hat{\bolddelta}_{\mathrm{dual}}) \\
    = &\left\| \nabla^2_\bolddelta \ell(\bar{\bolddelta})\left[\bolddelta^* - \hat{\bolddelta}_{\mathrm{dual}}\right]  \right\| - E_{q_n}(\hat{\bolddelta}_{\mathrm{dual}})
    \le R_2\left\| \bolddelta^* - \hat{\bolddelta}_{\mathrm{dual}}  \right\| - E_{q_n}(\hat{\bolddelta}_{\mathrm{dual}})  \text{ with probability at least } 1 - \epsilon_{R_2}, 
\end{align*}
where $\bar{\bolddelta}$ is a parameter that is in between $\hat{\bolddelta}$ and $\bolddelta^*$ in a coordinate-wise fashion. 

Combining the above inequality with \eqref{eq.dual.half} we obtain:
$\left\|\bolddelta^* - \hat{\bolddelta}_{\mathrm{dual}}  \right\| \le \frac{R_1+R_3}{R_2}\frac{1}{\sqrt{n_p \wedge n_q}} + \frac{C_0}{{R_2}}$ with probability at least $1 - \epsilon_{R_1} - \epsilon_{R_2}$ (union bound).
Since the $\Delta_n$ has a radius $\frac{R_1+R_3}{R_2}\frac{C_r}{\sqrt{n_p \wedge n_q}} + \frac{C_0}{R_2}$ where $C_r>1$, $ \hat{\bolddelta}_{\mathrm{dual}} $ is at the interior of $\Delta_{n}$. This means that even without the constraint $\bolddelta \in \Delta_n$, $\left\|\bolddelta^* - \hat{\bolddelta}_{\mathrm{dual}}  \right\| \le \frac{R_1+R_3}{R_2}\frac{1}{\sqrt{n_p \wedge n_q}} + \frac{C_0}{{R_2}}$ with probability at least $1 - \epsilon_{R_1} - \epsilon_{R_2}$.
\newpage

\subsection{Conditions for Consistency}
\label{sec.num.consistent}
In this experiment, we numerically evaluate the conditions required by Theorem \ref{them.2} and explore possible settings of the constants in this theorem. 

We let $\boldy_p = \boldy_q = [0,0]$, $l_p = l_q = \mathcal{N}_X(\boldx, \boldI), \boldx \in \mathbb{R}^2$ and $X_p, X_q$ be independently samples from a 2 dimensional standard normal distribution. $n_p = n_q = n$. 
We use $\boldf(\boldx) = \boldx$ in our posterior ratio model and 
it can be seen that $\bolddelta^* = \boldzero$ by our construction of the dataset.  

We first investigate the minimum eigenvalue condition on $\hessianLL$. 
In the left plot of Figure \ref{fig:consis.conditions}, we plot $\min_{\bolddelta \in \Delta_{n}} \lambda_\mathrm{min}\left[\nabla_\bolddelta^2\ell(\bolddelta)\right]$ (which is obtained by numerical optimization) as a red solid line. Here $\Delta_{n} :=  \mathrm{Ball}\left(\boldzero, \frac{10}{\sqrt{n}}\right)$ and the error bar indicates the standard deviation of 50 runs. It can be seen that when $n \ge 400$, $\min_{\bolddelta \in \Delta_{n}} \lambda_\mathrm{min}\left[\nabla_\bolddelta^2\ell(\bolddelta)\right]$ minus two standard deviation is safely bounded by 0.18 from below. 

We also investigate the tightness a lower-bound obtained from Wyel's inequality (see \eqref{eq.lower.hessian}) which is marked by a blue dash line. As we can see, when $n \ge 400$, Wyel's inequality becomes very close to $\min_{\bolddelta \in \Delta_{n}} \lambda_\mathrm{min}\left[\nabla_\bolddelta^2\ell(\bolddelta)\right]$ and consequently $\lambda_{\min}\left[\hessianLLstar\right]$ (black dash line) becomes a good approximation to $\min_{\bolddelta \in \Delta_{n}} \lambda_\mathrm{min}\left[\nabla_\bolddelta^2\ell(\bolddelta)\right]$. 

Now we explore some possible settings of the constants in Theorem \ref{them.2}. 
It can be seen from Theorem \ref{them.2} that $R_1$ and $R_2$ determine $\epsilon_{R_1}$ and $\epsilon_{R_2}$. To make probability 1 - $\epsilon_{R_1} - \epsilon_{R_2}$ as large as possible, we should make $R_1$ large but $R_2$ small. However, if the ratio $R_1/R_2$ is large, $\Delta_n$ would be large too ($C_r>1$) and condition \eqref{eq.mineig} may struggle to hold. To investigate the possible settings of $R_1$ and $R_2$, we plot $\frac{\|\sqrt{n}\nabla_\bolddelta \ell(\bolddelta^*)\|}{\inf_{\bolddelta\in\Delta_n} \lambda_{\min}\left[\hessianLL\right]}$ in the right plot of Figure \ref{fig:consis.conditions}, where $\Delta_n$ has a radius $10/\sqrt{n}$. The plot is generated from 500 repetitions with different draws of samples from $p(\boldx), q(\boldx)$. Error bars indicate standard deviations. We choose $C_r = 1.5$. 

Since the ratio $\frac{\|\sqrt{n}\nabla_\bolddelta \ell(\bolddelta^*)\|}{\inf_{\bolddelta\in\Delta_n} \lambda_{\min}\left[\hessianLL\right]} \le \frac{R_1}{R_2}= 10/C_r$, \emph{6.66 is the upper limit of } $\frac{\|\sqrt{n}\nabla_\bolddelta \ell(\bolddelta^*)\|}{\inf_{\bolddelta\in\Delta_n}}$. 
The plot shows, this requirement is comfortably met with some room to spare. 
We can see that the averaged ratio plus two standard deviations is nicely bounded by $10/C_r$ from above.

We can also see from this plot that $\frac{\|\sqrt{n}\nabla_\bolddelta \ell(\bolddelta^*)\|}{ \lambda_{\min}\left[\hessianLLstar\right]}$ becomes a nice approximation to $\frac{\|\sqrt{n}\nabla_\bolddelta \ell(\bolddelta^*)\|}{\inf_{\bolddelta\in\Delta_n} \lambda_{\min}\left[\hessianLL\right]}$ when $n$ is larger than 400, which aligns with our expectation in Section \ref{sec.eig.hessian} that $ \boldSigma \approx \hessianLLstar \approx  \hessianLL, \forall \bolddelta \in \Delta_n$ when $n$ is large. 

\begin{figure}
	\centering
	\includegraphics[width=.4\textwidth]{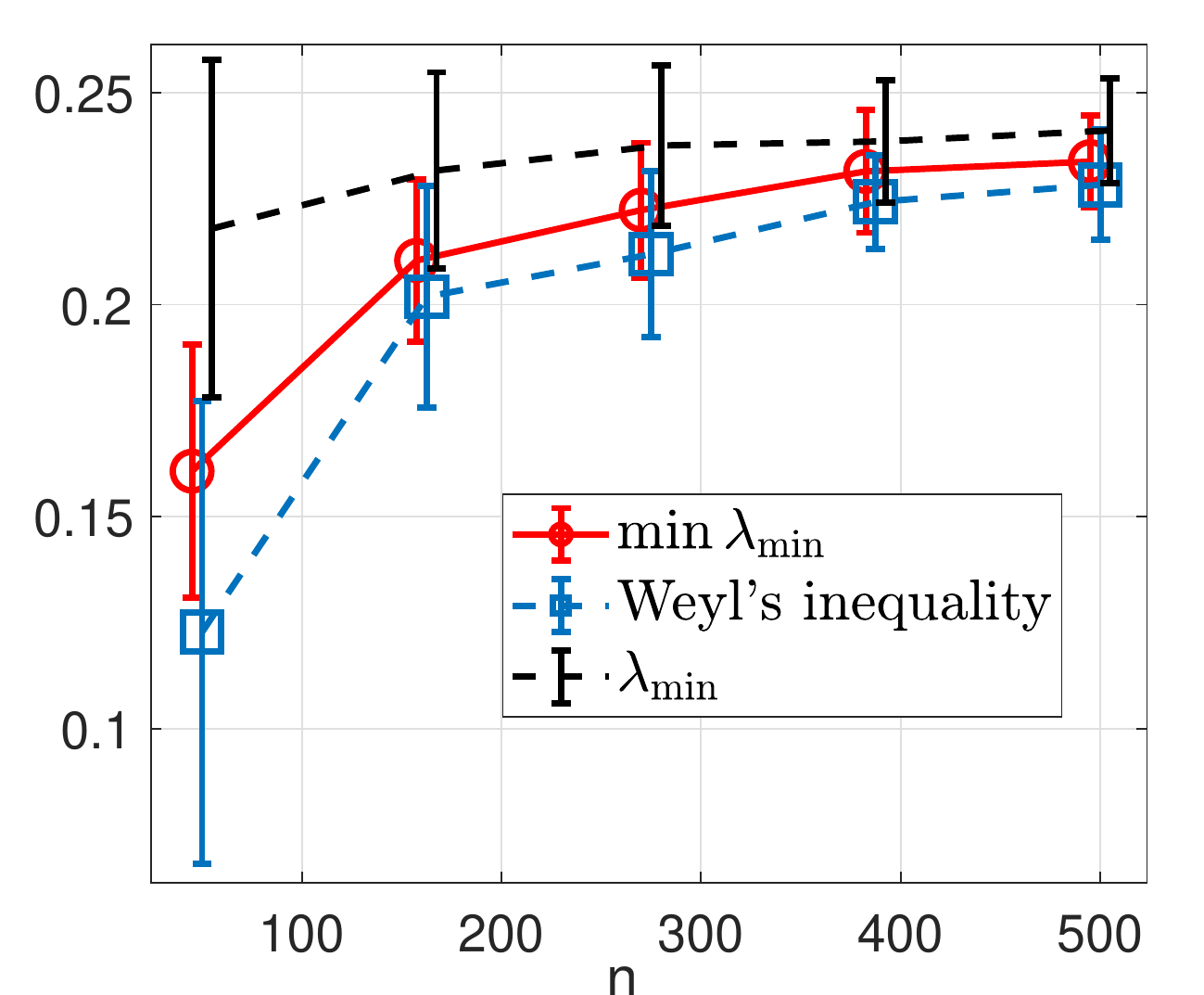}
	\includegraphics[width=.4\textwidth]{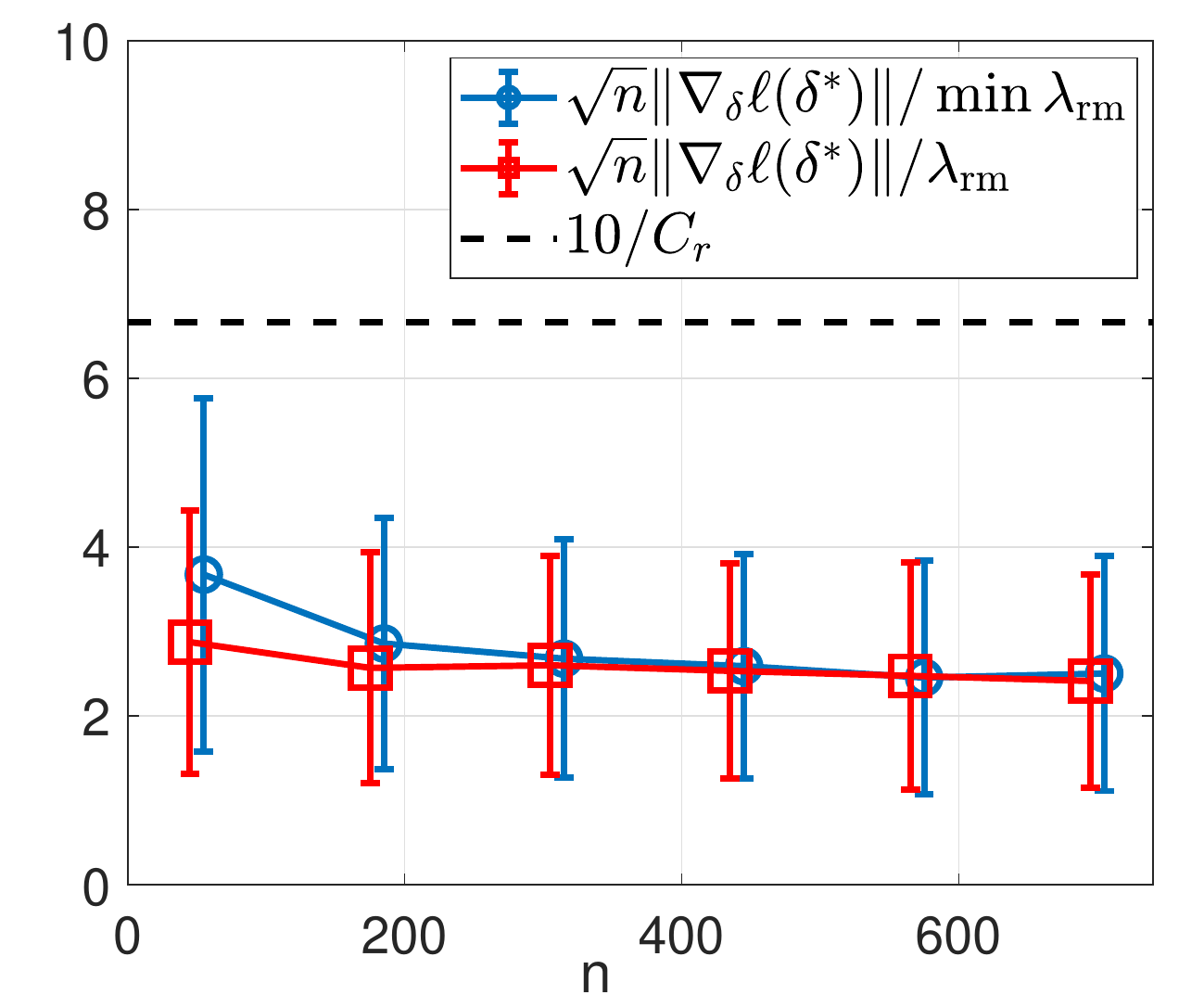}
	\caption{Numerical evaluation of Theorem \ref{them.2}'s conditions}
	\label{fig:consis.conditions}
\end{figure}

\newpage
\section{Additional Information on Latent Signal Detection}
In Section \ref{sec.latent.detection}, the state-space data experiment, we chose to use normal distribution as our likelihood function: $l_p(\boldx) := \mathcal{N}_{\boldy_p}(\boldx, \boldI)$ and $l_q(\boldx) := \mathcal{N}_{\boldy_q}(\boldx, \boldI)$. $\boldf(\boldx) := \mathrm{AR}(\boldx,20)$, where $\mathrm{AR}$ stands for the MALTAB ``autocorr'' function. In the speech keyword detection experiment, we used likelihood functions $l_p(\boldx) := \mathcal{N}_{\boldy_p}(\boldx, 0.18\cdot \boldI)$ and $l_q(\boldx) := \mathcal{N}_{\boldy_q}(\boldx, 0.18\cdot \boldI)$. $\boldf(\boldx) := [x_1x_2, x_2x_3, \dots, x_{d-1}x_{d}]$, i.e. the Lag-1 autocorrelation feature. 

The State-Space Transformation (SST) method \citep{Moskvina2003} produces a discrepancy measure by evaluating two trajectory matrices of two time-series signals using singular spectrum analysis. We use the first 20 eigenvectors to compare the difference between two spectrum. The number ``20'' was confirmed to be a reasonable choice in our preliminary experiments.

\section{Additional Information on Locally Linear Classifier Extraction}
In Section \ref{sec.locallinear}, we  estimated $\frac{p(\boldx|-1)}{p(\boldx|1)}$. However, if one wanted to produce a decision boundary that is plotted on the left, Figure \ref{fig:localclassifier}, one would need to know the constant $c$. In our experiment, we estimated such a constant using 400 data points in the original labelled dataset near $X_\mathrm{loc}$. In practice, this is not possible as we cannot possibly know any labelled data points near $X_{\mathrm{loc}}$. However, we argue that knowing $c$ is not relevant to our task, which is identifying useful features of a local linear classifier. Such features can be efficiently identified by looking at the magnitude of $|\hat{\delta}_i|$, which was what we did in the MNIST digits experiment.

\end{document}